\documentclass{article}
\usepackage[utf8]{inputenc}
\usepackage[english]{babel}
\usepackage[top=1in,bottom=1in,left=1in,right=1in]{geometry}

\usepackage[round]{natbib}

\usepackage{graphicx}
\usepackage{epstopdf}
\ifpdf
\usepackage{multirow}%
\usepackage{amsmath,amssymb,amsfonts,mathtools,amsthm}%
\usepackage{mathrsfs}%
\usepackage[title]{appendix}%
\usepackage{xcolor}%
\usepackage{textcomp}%
\usepackage{manyfoot}%
\usepackage{booktabs}%
\usepackage{listings}%
\usepackage{tikz}
\usetikzlibrary{positioning}
\usetikzlibrary{arrows}
\usepackage{bbm}
\usepackage{amsopn}
\usepackage{hyperref}

\newcommand{\R}{\mathbb R}

\newcommand{\N}{\mathbb N}
\newcommand{\Z}{\mathbb Z}

\newcommand{\E}{\mathbb E}

\newcommand{\sym}{\mathfrak{S}}
\newcommand{\comps}{\Lambda}
\newcommand{\compsk}{\lambda}
\newcommand{\kvec}{\Vec{k}}
\newcommand{\rvec}{\Vec{r}}
\newcommand{\one}{\mathbbm 1}

\newcommand{\ind}[1]{\mathbb{I}\left\{#1\right\}}

\newcommand{\multi}[2]{\big(\!\binom{#1}{#2}\!\big)}
\newcommand{\amulti}[2]{\bigg(\!\binom{#1}{#2}\!\bigg)}
\newcommand{\amultib}[2]{\Bigg(\!\binom{#1}{#2}\!\Bigg)}

\newcommand{\suml}{\sum_{\ell=1}^d}
\newcommand{\prodl}{\prod_{\ell=1}^d}
\newcommand{\cupl}{\cup_{\ell=1}^d}

\newcommand{\tvec}{\text{vec}}

\newcommand{\rank}{\text{rank}}

\newcommand{\rao}{\odot}
\newcommand{\sgn}{\mathrm{sgn}}
\newcommand{\ta}{\widetilde{A}}
\newcommand{\tb}{\widetilde{B}}

\newcommand{\ba}{\overline{A}}
\newcommand{\bb}{\overline{B}}

\newcommand{\bij}{\hookrightarrow\!\!\!\!\to}
\DeclareMathOperator{\diag}{diag}

\newtheorem{theorem}{Theorem}[section]
\newtheorem{lemma}[theorem]{Lemma}

\begin{document}

\title{Memory capacity of two layer neural networks with smooth activations\thanks{This work was partially funded by a UBC DSI Postdoctoral Fellowship and NSERC Discovery Grant No. 2021-03677.}}

\author{Liam Madden\thanks{Department of Electrical and Computer Engineering, University of British Columbia, Vancouver, BC, Canada.}
\and Christos Thrampoulidis\footnotemark[2]}

\date{\today}

\maketitle

\begin{abstract}
Determining the memory capacity of two layer neural networks with $m$ hidden neurons and input dimension $d$ (i.e., $md+2m$ total trainable parameters), which refers to the largest size of general data the network can memorize, is a fundamental machine learning question.
For activations that are real analytic at a point and, if restricting to a polynomial there, have sufficiently high degree, we establish a lower bound of $\lfloor md/2\rfloor$ and optimality up to a factor of approximately $2$. All practical activations, such as sigmoids, Heaviside, and the rectified linear unit (ReLU), are real analytic at a point. Furthermore, the degree condition is mild, requiring, for example, that $\binom{k+d-1}{d-1}\ge n$ if the activation is $x^k$. Analogous prior results were limited to Heaviside and ReLU activations---our result covers almost everything else.
In order to analyze general activations, we derive the precise generic rank of the network's Jacobian, which can be written in terms of Hadamard powers and the Khatri-Rao product.
Our analysis extends classical linear algebraic facts about the rank of Hadamard powers. Overall, our approach differs from prior works on memory capacity and holds promise for extending to deeper models and other architectures.
\end{abstract}

\section{Introduction}

Machine learning models, such as neural networks, have become increasingly adept at performing tasks, such as image classification and natural language processing, but also increasingly opaque in terms of their underlying mechanisms. And the intuition that experts employ to construct such models, while useful now, may be based on an erroneous understanding of those underlying mechanisms, much as a field specialist in an engineering firm may come up with explanations that are useful in the field but contradict the underlying physics. It is important that we understand the underlying mechanisms in machine learning so that we can mitigate inherent biases and avoid potential failures.

Neural networks, despite their simplicity, are still not fully understood. From the perspective of traditional optimization, they should be extremely hard to train due to the nonconvexity of their associated training objectives. However, even simple algorithms like stochastic gradient descent are able to find global minima to these training objectives \citep{du2019shallow}. At the heart of this phenomena is the fact that neural networks are expressive enough to interpolate data sets: interpolation is known to accelerate convergence \citep{vaswani2019fast}. But, from the perspective of traditional statistical learning theory then, they should not generalize well to unseen data. Nevertheless, they do generalize well to unseen data, interpolating without merely memorizing the seen data. In fact, it has been observed that, as the number of parameters increases, the generalization error decreases until it reaches a minimum, then increases until it reaches the interpolation threshold, then continues to decrease \citep{belkin2019reconciling}. Thus, the interpolation threshold is integral to understanding not only the expressivity of neural networks, but their convergence and generalization properties as well, and so we take it as the theme of the present paper.

The output of the two layer neural network model with data matrix $X=[x_1|\cdots|x_n]\in\R^{d\times n}$, first layer weight matrix $W=[w_1|\cdots|w_m]\in\R^{d\times m}$, first layer bias vector $b\in\R^m$, second layer weight vector $v\in\R^m$, and activation function $\psi:\R\to\R$ is
\begin{align*}
    h(W,b,v,X) =\sum_{i=1}^mv_i\psi\left(X^Tw_i+\one_nb_i\right)=\psi\left(X^TW+\one_n b^T\right)v\in\R^n,
\end{align*}
where $\psi$ is applied coordinate-wise and $\one_n\in\R^n$ is the vector of ones. The columns of $W$ and elements of $b$ correspond to $m$ hidden neurons, $\psi(W^TX+\one_n b^T)$ embeds the data into $\R^m$, $m$ is the embedding dimension, and $v$ corresponds to a linear model applied to the embedded data. For fixed $m$, the memory capacity of $h$ is the largest $n$ such that $h(\cdot,\cdot,\cdot,X)$ is surjective for generic $X$ \citep{cover1965geometrical}. As is common in the literature, e.g. in \cite{allman2009identifiability}, we use ``generic'' to mean that $X$ comes from a set whose complement is contained in the zero set of a non-identically zero real analytic function; in particular, this implies that the complement is measure zero and closed. Given $n$, we are interested in how large $m$ has to be in order for the memory capacity of $h$ to be at least $n$.

We consider general activations that are real analytic at a point, which only precludes pathological examples such as the Cantor function. The only requirement beyond this is that if the activation restricts to a polynomial, then the polynomial must have sufficiently high degree. Common examples of real analytic activations are sigmoids, such as the logistic function, tanh, and arctan, as well as smoothed rectified linear units, such as GELU. Examples of non-analytic activations that are, nevertheless, analytic at a point, are splines (piecewise polynomials), including ReLU, which is piecewise linear. On the other hand, to show that the memory capacity result is tight we have to assume that $\psi$ is continuously differentiable. While this excludes some splines, such as ReLU, the point of it is just to preclude the possibility that $h(\cdot,\cdot,\cdot,X)$ is a space-filling map, like Peano's curve.

To construct an (implicit) interpolating solution for generic $X$, we will (1) show that the Jacobian of $h(\cdot,b,v,X)$ has generically full rank near the point where $\psi$ is analytic, (2) scale the first layer weight matrix and use the bias vector to stay within the interval of convergence, (3) apply the Constant Rank Theorem to get an implicit interpolating solution for a scaled and shifted version of the neural network, and (4) show how to implement this solution by doubling the width of the original neural network. The transpose of the Jacobian of $h$ with respect to $\tvec(W)$ is $\diag(v)\psi'\left(W^TX+b\one_n^T\right)\rao X\in\R^{md\times n}$ where $\rao$ is the Khatri-Rao product (column-wise Kronecker product) \citep[Sec. VI.A]{oymak2020toward}.



\subsection{Results}

In order to facilitate adapting the steps (1)-(4) to more general models, we state and prove Theorem~\ref{thm:everything}, which, when combined with a full rank Jacobian, provides an (implicit) interpolating solution for general machine learning models, such as two layer neural networks and deep neural networks. Then, for two layer neural networks, we prove the following tight, up to a factor of $\approx 2$, memory capacity result:
\begin{itemize}
    \item Theorem~\ref{thm:ae}: If $md\ge 2n$ and $m$ is even, then $h(\cdot,\cdot,\cdot,X)$ is surjective for generic $X$.
    \item Theorem~\ref{thm:lowerbound}: If $m(d+2)< n$, then, for all $X$, $h(\cdot,\cdot,\cdot,X)$ is not surjective.
\end{itemize}

In order to prove Theorem~\ref{thm:ae}, we first have to derive the generic rank of the Jacobian. Let $A\in \R^{m\times d}$ and $B\in\R^{n\times d}$ and let $(c_k)$ be a sequence in $\R$ and $\phi:\R\to\R$ be non-polynomial real analytic. Then we prove the following results, where the exponent $(k)$ denotes the $k$th Hadamard power:
\begin{itemize}
\item Theorem~\ref{thm:damm}: The rank and Kruskal rank of $(AB^T)^{(k)}$ are $\min\{m,n,\binom{k+d-1}{k}\}$ for generic $(A,B)$.
\item Theorem~\ref{thm:poly}: The rank and Kruskal rank of
\begin{align*}
    \sum_{k=0}^K c_k \left(AB^T\right)^{(k)}\text{ are }\min\Bigg\{m,n,\sum_{k=0}^K \ind{c_k\neq 0}\binom{k+d-1}{k}\Bigg\}
\end{align*}
for generic $(A,B)$.
\item Theorem~\ref{thm:khatri}: The rank and Kruskal rank of $B^T\rao (AB^T)^{(k)}$ are $\min\{md,n,\binom{k+d}{k+1}\}$ for generic $(A,B)$.
\item Theorem~\ref{thm:polykhatri}: The rank and Kruskal rank of
\begin{align*}
    \min\Bigg\{md,n,\sum_{k=0}^K \ind{c_k\neq 0}\binom{k+d}{k+1}\Bigg\}\text{ are }B^T\rao \sum_{k=0}^K c_k \left(AB^T\right)^{(k)}
\end{align*}
for generic $(A,B)$.
\item Theorem~\ref{thm:real}: The rank and Kruskal rank of $\phi(AB^T)$ are $\min\{m,n\}$ for generic $(A,B)$.
\item Theorem~\ref{thm:realkhatri}: The rank and Kruskal rank of $B^T\rao \phi(AB^T)$ are $\min\{md,n\}$ for generic $(A,B)$.
\end{itemize}

It may be helpful to see the relevant objects in the following diagram, which has arrows indicating particular proof extensions.

\begin{center}
\begin{tikzpicture}[>=triangle 45,font=\sffamily]
    \node (damm) at (0,0) {Thm~\ref{thm:damm}: $\left(AB^T\right)^{(k)}$};
    \node (khatri) [right=4cm of damm] {Thm~\ref{thm:khatri}: $B^T\rao\left(AB^T\right)^{(k)}$};
    \node (poly) [below=2cm of damm] {Thm~\ref{thm:poly}: $\sum_{k=0}^K c_k\left(AB^T\right)^{(k)}$};
    \node (polykhatri) [below=2cm of khatri] {Thm~\ref{thm:polykhatri}: $B^T\rao \sum_{k=0}^K c_k\left(AB^T\right)^{(k)}$};
    \node (real) [below=2cm of poly] {Thm~\ref{thm:real}: $\phi(AB^T)$};
    \node (realkhatri) [below=2cm of polykhatri] {Thm~\ref{thm:realkhatri}: $B^T\rao \phi(AB^T)$};
    \draw [->] (damm) -- (khatri) node [midway,above] {1} node [midway,below] {Sec~\ref{sec:khatri}};
    \draw [->] (damm) -- (poly) node [midway,left] {2};
    \draw [->] (poly) -- (polykhatri) node [midway,above] {1};
    \draw [->] (khatri) -- (polykhatri) node [midway,left] {2} node [midway,right] {Sec~\ref{sec:polykhatri}};
    \draw [->] (poly) -- (real) node [midway,left] {3};
    \draw [->] (polykhatri) -- (realkhatri) node [midway,left] {3} node [midway,right] {Sec~\ref{sec:analytic}};
\end{tikzpicture}
\end{center}

Theorems \ref{thm:khatri}, \ref{thm:polykhatri}, and \ref{thm:realkhatri}
are our main contributions (in addition to Theorems~\ref{thm:everything}-\ref{thm:lowerbound}). Theorem~\ref{thm:damm} was proved in \cite{damm2023hadamard} but we give a novel proof that allows us to extend to the other results. The second proof extension, shown in the diagram by the arrows labeled ``2,'' follows almost immediately thanks to the structure of our proof. The first proof extension, shown in the diagram by the arrows labeled ``1,'' is the most involved and requires a novel decomposition of $B^T\rao (AB^T)^{(k)}$. The third proof extension, shown in the diagram by the arrows labeled ``3,'' is not difficult but subtle and involves reducing to the polynomial setting using Taylor's theorem.

\subsection{Related work}
\label{sec:related}
For simplicity, let the bias vector be zero. Then the transpose of the Jacobian of $h$ with respect to $\tvec(W)$ is $\Phi=\diag(v)\psi'\left(W^TX\right)\rao X$. If $\rank(\Phi)=n$ at $(W_0,0,v_0,X)$, then, given a label vector $y\in\R^n$, we can solve for $U_0\in\R^{d\times m}$ in the linear equation
\begin{align}
\label{eq:bubeck}
\begin{split}
    y &= \left(\diag(v_0)\psi'\left(W_0^TX\right)\rao X\right)^T\tvec(U_0)\\
    \text{and use }&=\lim_{\epsilon\to 0}~\psi\left(X^T[W_0+\epsilon U_0~ W_0]\right)[v_0/\epsilon;-v_0/\epsilon]
\end{split}
\end{align}
to show that the image of $h(\cdot,0,\cdot,X)$ with embedding dimension $2m$ is dense \citep[Prop. 4]{bubeck2020network}. \cite{bubeck2020network} used this trick to show that if $\psi$ is ReLU (i.e. $\psi(x)=\max\{0,x\})$, then $m\ge 4\lceil n/d\rceil$ is sufficient to guarantee that, for all $X$ in general linear position, $h(\cdot,0,\cdot,X)$ is surjective (since in this case we don't have to take $\epsilon$ all the way to zero). \cite{baum1988multilayer} had already proved a similar result if $\psi$ is the Heaviside function (i.e. $\psi(x)=\ind{x\ge 0}$) and the image of $h(\cdot,0,\cdot,X)$ only needs to contain $\{0,1\}^n$, and \cite{yun2019small} had extended this to ReLU.

Recently, \cite{zhang2021when} used the Inverse Function Theorem to generalize Eq.~\eqref{eq:bubeck}, demonstrating that if $\psi$ is smooth and if $\rank(\Phi)=n$ at some $(W_0,0,v_0,X)$, then $h(\cdot,0,\cdot,X)$ with embedding dimension $2m$ is surjective. Their approach inspired our Theorem~\ref{thm:everything}, which further extends their result to encompass a broader range of machine learning models.


Regarding the conditions that ensure $\rank(\Phi)=n$ for smooth activations,  it is shown in Corollary 3.1 of \cite{montanari2022interpolation} that, given $n\in\N$, there exists $C(n)>0$ such that $md/\log^{C(n)}(md)\ge n$ is sufficient to guarantee that, with high probability over random $X$, $h(\cdot,0,\cdot,X)$ is surjective. Ignoring the dependence of $C(n)$ on $n$, this bound has extra log factors and only holds for $X$ with high probability. Thus, our results represent an improvement in both these aspects. To better understand the basis for these improvements, it's helpful to delve into the origins of their result, which is actually a corollary of an intermediate step in the proof of a convergence result, as elaborated below. 
In order to prove that $\rank(\Phi)=n$, they  show that the smallest eigenvalue of $\Phi^T\Phi$ is positive by bounding the minimum eigenvalue of $\E_{W,v}[\Phi^T\Phi]$ away from zero and applying the matrix Chernoff inequality. For example, letting $\lambda$ denote $\lambda_{\min}(\E_{W,v}[\Phi^T\Phi/m])$, then applying Remark 5.3 of \cite{tropp2012user} gives the following: if
\begin{align*}
    m \ge \frac{2\|X\|_2^2\log(n/\delta)}{\lambda}
\end{align*}
then, with probability greater than $1-\delta$,
\begin{align*}
    \lambda_{\min}(\Phi^T\Phi/m) \ge \lambda - \frac{\|X\|_2}{\sqrt{m}}\sqrt{2\lambda\log(n/\delta)}.
\end{align*}
This bound on $\lambda_{\min}(\Phi^T\Phi/m)$ was first proved in \cite{oymak2020toward}, a paper which, along with \cite{song2019quadratic}, tightened the analysis of the landmark paper \cite{du2019shallow}. Subsequently, \cite{montanari2022interpolation} established a  lower bound on $\lambda$, which aligns with empirical investigations outlined in \cite{xie2017diverse}. But, while bounding the distance of $\Phi^T\Phi$ from $\E_{W,v}[\Phi^T\Phi]$ ends up being useful for the convergence analysis, it is unnatural for memory capacity analysis: there is no reason why $\Phi^T\Phi$ needs to be close to $\E_{W,v}[\Phi^T\Phi]$ in order for $\rank(\Phi)=n$, since, as we show, $\rank(\Phi)=n$ generically.

For the rank results without the Khatri-Rao product, while \cite{damm2023hadamard} recently derived the generic rank of $(AB^T)^{(k)}$, it seems the generic rank of $\sum_{k=0}^Kc_k (AB^T)^{(k)}$ has actually been known for some time, though it is difficult to find precise statements of equality rather than upper bounds. For example, Lemma 2.3 of \cite{alon2009perturbed} proved in four lines that the rank of $(AB^T)^{(k)}$ and $\sum_{k=0}^Kc_k (AB^T)^{(k)}$ are bounded by $\binom{k+d-1}{k}$ and $\binom{K+d}{K}$ respectively. Lemma 4.4 of \cite{barvinok2012approximations} proved the latter bound, calling it a ``standard linear algebra fact.'' Moreover, if we slightly modify their proof by including the coefficients $\ind{\alpha_m\neq 0}$ in the final expression, then a little more work proves our Theorem~\ref{thm:poly}. We include this alternative proof in Appendix~\ref{sec:directsum} in order to explain why it cannot be extended to Khatri-Rao products.

Finally, for the rank results involving the Khatri-Rao product, the only previous result, other than corollaries of convergence results, was Lemma E.1 of \cite{zhang2021when}, which applies to non-polynomial real analytic functions. Upon a careful examination of their proof, we encountered inconsistencies that appeared unresolvable. Only after the initial release of our preprint, a correspondence with the authors of \cite{zhang2021when} revealed that these inconsistencies stemmed from typos and convoluted presentation of their methodology. During this discussion, we realized that the idea behind their proof is correct and can, in fact, be presented in a rather transparent manner, which we do in Appendix~\ref{sec:zhang} for completeness. 
Overall, our proof technique is independent and entirely distinct from theirs. Furthermore, our results apply to more general functions, such as polynomials and functions that are only real analytic at a point. The proof in \cite{zhang2021when} can be extended to polynomials as well, as we show in Appendix~\ref{sec:zhang}, but it provides a loose lower bound on the generic rank rather than giving the precise generic rank. In particular, it can only be used to justify using polynomials with an extremely large number of nonzero coefficients.


\subsection{Organization}

Our paper is organized as follows. In Section~\ref{sec:prelims}, we go through the necessary preliminaries. In Section~\ref{sec:warmup}, we present our main proof idea and prove Theorems~\ref{thm:damm} and~\ref{thm:poly}. In Section~\ref{sec:main}, we prove Theorems~\ref{thm:khatri}-\ref{thm:realkhatri}. In Section~\ref{sec:neuralprelims}, we prove Theorems~\ref{thm:everything} and~\ref{thm:ae}. In Section~\ref{sec:conclusion}, we conclude. In the appendix we go into further details about some of the related work.

\section{Preliminaries}
\label{sec:prelims}

In this section we discuss the necessary preliminaries. In Section~\ref{sec:rankprelims} we define the different notions of rank and discuss the linear algebra tools that we will employ. In Section~\ref{sec:analyticprelims}, we discuss some notions from the theory of analytic functions of several variables. In Section~\ref{sec:compositions}, we discuss some basic facts about weak compositions, which we will index some of our matrices by.

Throughout the present paper, the following notation will be used: $\Z_+$ denotes $\N\cup\{0\}$, $[n]$ denotes $\{1,2,\ldots,n\}$, $\binom{A}{n}$ denotes $\{B\subset A\mid |B|=n\}$, $\circ$ denotes the Hadamard product (and function composition), $\rao$ denotes the Khatri-Rao product defined as $[a_1|\cdots|a_d]\rao[b_1|\cdots|b_d]=[a_1\otimes b_1|\cdots|a_d\otimes b_d]$, the exponent $(k)$ denotes the $k$th Hadamard power, and ``Hadamard function'' refers to univariate functions applied coordinate-wise.

Let $A\in\R^{m\times n}$. We use $a_i$ to denote the $i$th column of $A$ and $a_{i,j}$ to denote the $(i,j)$th entry of $A$. Given $I\subset [m]$, and $J\subset [n]$, we denote the submatrix formed from the rows $I$ and columns $J$ of $A$ as $A_{I,J}$. We use $\sym_n$ to denote the symmetric group of degree $n$. Note that we can write $\sym_n=\{\sigma:[n]\bij [n]\}$. More generally, consider a matrix $A$ with rows indexed by a set $I$ with $|I|=m$ and columns indexed by a set $J$ with $|J|=n$. We order $I$ and $J$ lexicographically if they consist of pairs, i.e. the first entry in the pair changes slower than the second entry. If $D\in\R^{n\times n}$ is a diagonal matrix, then we use $D_{\ell}$ to denote $d_{\ell,\ell}$ and so avoid confusion with the $d$ used in $\R^d$.

Finally, we say a property holds for generic $x$ if it holds for all $x\in\R^n$ except on the zero set of a non-identically zero real analytic function.

\subsection{Rank of a matrix}
\label{sec:rankprelims}

Let $A\in\R^{m\times n}$. The rank of $A$ is the dimension of its column space or, equivalently, its row space. Let $I\subset [m]$ and $J\subset [n]$ with $|I|=|J|\le \min\{m,n\}$. Then the $(I,J)$ minor of $A$, denoted $\det_{I,J}(A)$, is the determinant of the submatrix $A_{I,J}$. The order of the minor is $|I|$. Using minors, the rank of $A$ can be equivalently defined as the largest $r$ such that there is a nonzero minor of order $r$. This allows us to write the following rank condition:
\begin{align*}
    \rank(A) \ge r\iff \sum_{\substack{(I,J)\subset [m]\times [n]\\|I|=|J|=r}}\det_{I,J}(A)^2\neq 0.
\end{align*}
Note that
\begin{align*}
    \sum_{\substack{(I,J)\subset [m]\times [n]\\|I|=|J|=r}}\det_{I,J}(A)^2
\end{align*}
is a multivariate polynomial in the entries of $A$. If we replace $A$ with a matrix-valued real analytic function of several variables $A(x)$, then the composition of the polynomial with $A$ is a real analytic function. We call this composite function the ``rank condition function of order $r$ associated with $A$''. Thus, we have that $\rank(A(x))\ge r$ if and only if $x$ is not in the zero set of the rank condition function of $r$ associated with $A$. If the zero set is not the whole domain (i.e. if the rank condition function is not identically zero), then it is very limited as we will explain in Section~\ref{sec:analyticprelims} (in particular, it is measure zero and closed). So, our task lies in showing that the rank condition function is not identically zero. To do so we use Leibniz's determinant formula \citep[Def. 10.33]{axler2015linear},
\begin{align*}
    \det(A) = \sum_{\sigma\in\sym_n}\sgn(\sigma)\prod_{i=1}^n a_{i,\sigma(i)}=\sum_{\sigma\in\sym_n}\sgn(\sigma)\prod_{i=1}^n a_{\sigma(i),i}\text{ for }A\in\R^{n\times n},
\end{align*}
and the Cauchy-Binet formula \citep[Sec. I.2.4]{gantmacher1960matrices},
\begin{align*}
    \det(AB^T) = \sum_{S\in\binom{[n]}{m}} \det_{[m],S}(A)\det_{[m],S}(B)\text{ for }A,B\in\R^{m\times n}~(n\ge m).
\end{align*}
Specifically, we use the following corollary of the Cauchy-Binet formula.

\begin{lemma}
\label{lma:binetdiag}
Let $A\in\R^{m\times N}$ and $B\in\R^{n\times N}$, and let $D\in\R^{N\times N}$ be diagonal.
Let $I\subset [m]$ and $J\subset [n]$ such that $|I|=|J|=s\le \min\{m,n\}$.
Then
\begin{align*}
    \det_{I,J}\left(ADB^T\right) = \sum_{S\in\binom{[N]}{s}}\left(\prod_{\ell\in S} D_{\ell}\right)\det_{I,S}(A)\det_{J,S}(B).
\end{align*}
\end{lemma}
\begin{proof}
Applying the Cauchy-Binet formula, we get
\begin{align*}
    \det_{I,J}\left(ADB^T\right) &= \sum_{S\in\binom{[N]}{s}}\det_{I,S}(A)\det_{S,J}\left(DB^T\right)\\
    &= \sum_{S,S'\in\binom{[N]}{s}}\det_{I,S}(A)\det_{S,S'}(D)\det_{J,S'}(B)
\end{align*}
but, since $D$ is diagonal, $\det(D_{S,S'})=0$ unless $S=S'$, in which case it equals $\prod_{\ell\in S}D_{\ell}$ and so the result follows.
\end{proof}

Note that if $s>N$, then $\binom{[N]}{s}$ is empty and so the minor is identically zero.

These are powerful tools because they can be used to write the minor of a matrix-valued polynomial as a sum of polynomials. Thus, the question of rank becomes a question of the linear independence of multivariate polynomials.

In addition to the usual notion of rank, there are stronger notions of rank. For example, the Kruskal rank of $A$ is the largest $r$ such that every subset of $r$ columns of $A$ is linearly independent \citep{kruskal1977three}. Given $J\subset [n]$ such that $|J|=r\le m$, the columns of $A_{[m],J}$ are linearly independent if and only if the rank of $A_{[m],J}$ is $r$. Thus, we have the following Kruskal rank condition:
\begin{align*}
    \text{the Kruskal rank of }A\ge r\iff \prod_{\substack{J\subset [n]\\|J|=r}}\sum_{\substack{I\subset [m]\\|I|=r}}\det_{I,J}(A)^2\neq 0.
\end{align*}

\subsection{Real analytic functions of several variables}
\label{sec:analyticprelims}

We will use properties of analytic functions to (1) extend our results from polynomial $A$ to real analytic $A$ and (2) to interpret the rank being deficient on the zero set of a non-identically zero analytic function.

First, an analytic function is uniquely determined by the coefficients of its power series expansion at each point. This is a corollary of Taylor's theorem, which gives these coefficients in terms of the partial derivatives (of all orders) of the function at a point. Thus, to show that the rank condition associated with $A$ is not identically zero, we only have to show that a subset of its coefficients are not all zero (since \textit{all} of the coefficients of the zero function are zero).

Second, the zero set of a non-identically zero analytic function is measure zero by Corollary 10 of \cite{gunning2022analytic}, a corollary of the identity theorem \citep[Thm. 6]{gunning2022analytic} and Jensen's inequality \citep[Thm. 9]{gunning2022analytic}. The identity theorem states that if $f$ and $g$ are analytic on an open connected domain and $f=g$ on a nonempty open subset, then $f=g$ on the whole domain. The identity theorem and its corollary hold for real analytic functions as well (this is clear from the proofs), so we have shown the measure zero part of the following proposition. To show the closed part, just use that the preimage of $\R\backslash\{0\}$ is open since $f:\R^n\to\R$ is continuous.

\begin{lemma}
\label{prop:generic}
The zero set of a non-identically zero real analytic function is measure zero and closed.
\end{lemma}

Actually, more can be said. The zero set of a non-identically zero real analytic function is locally a finite union of lower dimensional manifolds. In particular, this implies that it is measure zero and closed, but it also allows us to visualize its complement as a countable union of (possibly unbounded) cells \citep{guaraldo1986topics}.

\subsection{Weak compositions}
\label{sec:compositions}

We will index our decompositions by weak compositions. There is a well-known bijection between multisets of cardinality $k$ taken from $[d]$---$i_1,\ldots,i_k\in[d]$ such that $i_1\leq\ldots\le i_k$---and $d$-tuples of non-negative integers whose sum is $k$, i.e. weak compositions of $k$ into $d$ parts---$k_1,\ldots,k_d\in\Z_+$ such that $k_1+\cdots+k_d=k$. It can be seen by viewing $k_i$ as the multiplicity of $i$, i.e. $k_i=|\{\ell\in[k]\mid i=i_{\ell}\}|$. The total number of either is $\multi{d}{k}\coloneqq\binom{k+d-1}{k}=\binom{k+d-1}{d-1}$ by stars and bars \citep[Thm. 2.5.1]{brualdi2010combinatorics}.

Let
\begin{align*}
    \compsk(k) &\coloneqq \Big\{\kvec\in \Z_+^d\mid k_1+\cdots+k_d=k\Big\}\\
    \text{and }\comps(K) &\coloneqq \Big\{\kvec\in \Z_+^d\mid k_1+\cdots+k_d\le K\Big\}.
\end{align*}
Note that
\begin{align*}
    |\comps(K)|=\sum_{k=0}^K\amulti{d}{k}=\sum_{k=0}^K\binom{k+d-1}{d-1}=\sum_{i=d-1}^{K+d-1}\binom{i}{d-1}=\binom{K+d}{d}
\end{align*}
by Zhū's Theorem \citep[Thm. 1.5.2]{merris2003combinatorics}, which was published by Zhū Shìjié in 1303 \citep{chu1303}.

\section{Results involving Hadamard powers}
\label{sec:warmup}

Here we present the main idea that allows us to derive the generic rank of $B^T\rao \phi(AB^T)$. In Section~\ref{sec:mainidea}, we show how to derive the generic rank of a matrix-valued real analytic function if we can first decompose it in an appropriate way. As a warm-up, we derive the generic rank of $AB^T$ in Section~\ref{sec:product}. Then, we derive the generic ranks of $(AB^T)^{(k)}$ and $\sum_{k=0}^Kc_k(AB^T)^{(k)}$ in Sections~\ref{sec:damm} and~\ref{sec:poly} respectively.

\subsection{Proof sketch}
\label{sec:mainidea}

Here we sketch how to derive the generic rank of a matrix-valued real analytic function. Let $f:\R^d\to \R^{m\times n}$ be real analytic. For all $I\subset [m]$ and $J\subset [n]$ such that $|I|=|J|\le \min\{m,n\}$, define
\begin{align}
\label{eq:gij}
\begin{split}
    g_{I,J}:\R^d&\to \R\\
    x&\mapsto \det_{I,J}(f(x)).
\end{split}
\end{align}
Let $g$ be a product of $g_{I,J}$'s or a product of sums of $g_{I,J}$'s. Then $g$ is real analytic. We would like to show that it is not identically zero so that we can apply Lemma~\ref{prop:generic}.

To start with, consider polynomial $f$. Then we want to find an appropriate decomposition $f(x)=A(x)DB(x)^T$ where $A$ and $B$ are matrix-valued real analytic functions and $D\in\R^{N\times N}$ is a diagonal matrix with no zeros.
Let $|I|=|J|=s\le \min\{m,n,N\}$ and enumerate $I$ as $\{i_1,\ldots,i_s\}$ and $J$ as $\{j_1,\ldots,j_s\}$.
Then, by the Cauchy-Binet formula,
\begin{align*}
    \det_{I,J}\left(A(x)DB(x)^T\right) &= \sum_{S\in\binom{[N]}{s}}\left(\prod_{\ell\in S} D_{\ell}\right)\det_{I,S}(A(x))\det_{J,S}(B(x))\\
    &\coloneqq \sum_{S\in\binom{[N]}{s}}\left(\prod_{\ell\in S} D_{\ell}\right)p_S(x),
\end{align*}
and by Leibniz's determinant formula,
\begin{align*}
    p_S(x) &= \left(\sum_{\sigma:I\bij S}\sgn(\sigma)\prod_{t=1}^s a_{i_t,\sigma(i_t)}(x)\right)\left(\sum_{\tau:J\bij S}\sgn(\tau)\prod_{t=1}^s b_{j_t,\tau(j_t)}(x)\right)\\
    &\coloneqq p_{1,S}(x)p_{2,S}(x)\\
    &=\sum_{\sigma:I\bij S}\sum_{\tau:J\bij S}\sgn(\sigma)\sgn(\tau)\prod_{t=1}^s a_{i_t,\sigma(i_t)}(x)b_{j_t,\tau(j_t)}(x).
\end{align*}
We call the monomials corresponding to the individual terms in each sum the ``Leibniz monomials'' of $p_{1,S}$, $p_{2,S}$, and $p_S$ respectively. Note that when we collect terms, the coefficient of a particular Leibniz monomial may or may not be zero.

We want to show that the non-identically zero $p_S$ are linearly independent and that there is at least one of them. One way we will do this is by showing that, for each $p_S$, there is a monomial with nonzero coefficient in $p_S$ and zero coefficient in all other $p_{S'}$.

If the non-identically zero $p_S$ are linearly independent and there is at least one of them, then $g_{I,J}$ is identically zero if and only if $\prod_{\ell \in S}D_{\ell}=0$ for all $S$ such that $p_S$ is not identically zero, which can't happen if $D$ has no zeros. Thus, $g$ equal to the product of all $g_{I,J}$ such that $|I|=|J|\le \min\{m,n,N\}$ is not identically zero. And so, all order $\min\{m,n,N\}$ or less minors of $f(x)$ are nonzero for generic $x$. On the other hand, all minors of order $>N$ are zero for arbitrary $x$. Thus, the rank and Kruskal rank of $f(x)$ are precisely $\min\{m,n,N\}$ for generic $x$.

Now, consider the setting where $f$ is not a polynomial. Since $g_{I,J}$ is real analytic, it is uniquely determined around zero by its power series expansion there (by Taylor's theorem). For any $K\in\N$, let $g_{I,J}^K$ be the sum of the monomials in the Taylor series of $g_{I,J}$ that we get by truncating the Taylor series of $f$ to only degree $K$ monomials. Then we can take $K$ as high as we want and reduce to the corresponding result for polynomial $f$.

\subsection{Matrix products}
\label{sec:product}

As a warm-up, we will derive the generic rank of matrix products. Let $A\in\R^{m\times d}$ and $B\in\R^{n\times d}$. Let $I\subset [m]$ and $J\subset [n]$ such that $|I|=|J|=s\le \min\{m,n,d\}$. Enumerate $I$ as $\{i_1,\ldots,i_s\}$ and $J$ as $\{j_1,\ldots,j_s\}$. We want to show that the $(I,J)$ minor of $AB^T$ is not identically zero.
By the Cauchy-Binet formula,
\begin{align*}
    p(A,B) &\coloneqq \det_{I,J}\left(AB^T\right)\\
    &= \sum_{S\in\binom{[d]}{s}}\det_{I,S}(A)\det_{J,S}(B)\\
    &\coloneqq \sum_{S\in\binom{[d]}{s}}p_S(A,B),
\end{align*}
and by Leibniz's determinant formula,
\begin{align*}
    p_S &= \sum_{\sigma:I\bij S}\sum_{\tau:J\bij S}\sgn(\sigma)\sgn(\tau)\prod_{t=1}^s a_{i_t,\sigma(i_t)}b_{j_t,\tau(j_t)}.
\end{align*}
The $(\sigma,\tau)$ Leibniz monomial is the only one that involves $a_{i_1,\sigma(i_1)},\ldots,a_{i_s,\sigma(i_s)}$ and $\\b_{j_1,\tau(j_1)},\ldots,b_{j_s,\tau(j_s)}$ and so, after collecting terms, it has coefficient $\sgn(\sigma)\sgn(\tau)\neq 0$ in $p_S$. Moreover, if the $(\sigma,\tau)$ monomial has nonzero coefficient in $p_{S'}$, then there must exist $\sigma':I\bij S'$ and $\tau':J\bij S'$ such that the $(\sigma,\tau)$ monomial in $p_S$ equals the $(\sigma',\tau')$ monomial in $p_{S'}$, which implies $S=\sigma(I)=\sigma'(I)=S'$. Thus, for each $p_S$, there is a monomial with nonzero coefficient in $p_S$ and zero coefficient in all other $p_{S'}$. In particular, this implies that the list of polynomials $(p_S)$ is linearly independent. Thus, since $p$ is the sum of $(p_S)$, $p$ is not identically zero and so all order $\min\{m,n,d\}$ or less minors of $AB^T$ are nonzero for generic $(A,B)$. On the other hand, all minors of order $>d$ are zero for arbitrary $(A,B)$. Thus, the rank and Kruskal rank of $AB^T$ are precisely $\min\{m,n,d\}$ for generic $(A,B)$. We discuss further implications of this result in Appendix~\ref{sec:directsum}.

\subsection{Hadamard powers}
\label{sec:damm}

In order to derive the generic rank of $\left(AB^T\right)^{(k)}$ we will decompose it and then follow steps similar to the previous section.

\begin{theorem}
\label{thm:damm}
Let $(A,B)\in\R^{m\times d}\times\R^{n\times d}$. Let $k\in\N$. Then the rank and Kruskal rank of $(AB^T)^{(k)}$ are $\min\{m,n,\multi{d}{k}\}$ for generic $(A,B)$.
\end{theorem}
\begin{proof}
First,
\begin{align*}
    \left(AB^T\right)^{(k)} &= \left(\sum_{i=1}^d a_ib_i^T\right)^{(k)}\\
    &=\sum_{i_1,\ldots,i_k=1}^d \left(a_{i_1}b_{i_1}^T\right)\circ\cdots\circ\left(a_{i_k}b_{i_k}^T\right)\\
    &=\sum_{i_1,\ldots,i_k=1}^d (a_{i_1}\circ \cdots\circ a_{i_k})(b_{i_1}\circ \cdots\circ b_{i_k})^T\\
    &\overset{*}{=}\sum_{\kvec\in \compsk(k)}\binom{k}{k_1,\ldots,k_d}\left(a_1^{(k_1)}\circ\cdots\circ a_d^{(k_d)}\right)\left(b_1^{(k_1)}\circ\cdots\circ b_d^{(k_d)}\right)^T\\
    &\coloneqq \ta D \tb^T
\end{align*}
where $A=[a_1|\cdots|a_d]$, $B=[b_1|\cdots|b_d]$, $(*)$ follows from the commutativity of the Hadamard product, the columns of $\ta$ are the unique Hadamard products of the columns of $A$, the columns of $\tb$ are the unique Hadamard products of the columns of $B$, and $D$ is a diagonal matrix with the corresponding multinomial coefficients. Note that
\begin{align*}
    \ta = \left[\prodl a_{i,\ell}^{k_{\ell}}\right]_{(i,\kvec)\in [m]\times \compsk(k)}\text{ and }\tb = \left[\prodl b_{i,\ell}^{k_{\ell}}\right]_{(i,\kvec)\in [n]\times \compsk(k)}.
\end{align*}
The inner dimension is $|\compsk(k)|=\multi{d}{k}$.
Let $I\subset [m]$ and $J\subset [n]$ such that $|I|=|J|=s\le \min\{m,n,\multi{d}{k}\}$. Enumerate $I$ as $\{i_1,\ldots,i_s\}$ and $J$ as $\{j_1,\ldots,j_s\}$.
We want to show that the $(I,J)$ minor is not identically zero.
By the Cauchy-Binet formula,
\begin{align*}
    p(A,B) &\coloneqq \det_{I,J}\left(\ta D\tb^T\right)\\
    &= \sum_{S\in\binom{\compsk(k)}{s}}\left(\prod_{\ell\in S} D_{\ell}\right)\det_{I,S}\left(\ta\right)\det_{J,S}\left(\tb\right)\\
    &\coloneqq \sum_{S\in\binom{\compsk(k)}{s}}\left(\prod_{\ell\in S} D_{\ell}\right)p_S(A,B),
\end{align*}
and by Leibniz's determinant formula,
\begin{align*}
    p_S &= \sum_{\sigma:I\bij S}\sum_{\tau:J\bij S}\sgn(\sigma)\sgn(\tau)\prod_{t=1}^s\prod_{\ell=1}^d a_{i_t,\ell}^{\sigma(i_t)_{\ell}}b_{j_t,\ell}^{\tau(j_t)_{\ell}}.
\end{align*}

Consider the $(\sigma,\tau)$ Leibniz monomial:
\begin{align*}
\begin{matrix}
a_{i_1,1}^{\sigma(i_1)_{1}}&\cdots&a_{i_1,d}^{\sigma(i_1)_{d}}\\
\vdots&\ddots&\vdots\\
a_{i_s,1}^{\sigma(i_s)_{1}}&\cdots&a_{i_s,d}^{\sigma(i_s)_{d}}
\end{matrix}\quad
\begin{matrix}
b_{j_1,1}^{\tau(j_1)_{1}}&\cdots&b_{j_1,d}^{\tau(j_1)_{d}}\\
\vdots&\ddots&\vdots\\
b_{j_s,1}^{\tau(j_s)_{1}}&\cdots&b_{j_s,d}^{\tau(j_s)_{d}}
\end{matrix}.
\end{align*}
Suppose it is equal to the $(\sigma',\tau')$ Leibniz monomial. Then $\sigma(i_t)=\sigma'(i_t)$ and $\tau(j_t)=\tau'(j_t)$ for each $t\in[s]$. In other words, $\sigma=\sigma'$ and $\tau=\tau'$. Thus, collecting terms, the monomial $\prod_{t=1}^s\prod_{\ell=1}^d a_{i_t,\ell}^{\sigma(i_t)_{\ell}}b_{j_t,\ell}^{\tau(j_t)_{\ell}}$ has coefficient $\sgn(\sigma)\sgn(\tau)\neq 0$ in $p_S$.

Moreover, if the $\prod_{t=1}^s\prod_{\ell=1}^d a_{i_t,\ell}^{\sigma(i_t)_{\ell}}b_{j_t,\ell}^{\tau(j_t)_{\ell}}$ monomial has nonzero coefficient in $p_{S'}$, then there must exist $\sigma':I\bij S'$ and $\tau':J\bij S'$ such that it equals $\prod_{t=1}^s\prod_{\ell=1}^d a_{i_t,\ell}^{\sigma'(i_t)_{\ell}}b_{j_t,\ell}^{\tau'(j_t)_{\ell}}$, which implies $S=\sigma(I)=\sigma'(I)=S'$.

Thus, for each $p_S$, there is a monomial, namely $\prod_{t=1}^s\prod_{\ell=1}^d a_{i_t,\ell}^{\sigma(i_t)_{\ell}}b_{j_t,\ell}^{\tau(j_t)_{\ell}}$, with nonzero coefficient in $p_S$ and zero coefficient in all other $p_{S'}$. In particular, this implies that the list of polynomials $(p_S)$ is linearly independent. Thus, since $p$ is a linear combination of $(p_S)$ such that the coefficients in the combination are nonzero (the coefficients are products of multinomial coefficients), $p$ is not identically zero and so all order $\min\{m,n,\multi{d}{k}\}$ or less minors of $(AB^T)^{(k)}$ are nonzero for generic $(A,B)$. On the other hand, all minors of order $>\multi{d}{k}$ are zero for arbitrary $(A,B)$. The conclusion follows.
\end{proof}

\subsection{Polynomial Hadamard functions}
\label{sec:poly}

The proof of the previous section easily extends to $\sum_{k=0}^Kc_k(AB^T)^{(k)}$.

\begin{theorem}
\label{thm:poly}
Let $(A,B)\in\R^{m\times d}\times\R^{n\times d}$. Let $K\in \N$ and let $(c_k)$ be a sequence in $\R$. Then the rank and Kruskal rank of
\begin{align*}
    \sum_{k=0}^K c_k \left(AB^T\right)^{(k)}
\end{align*}
are
\begin{align*}
    \min\Bigg\{m,n,\sum_{k=0}^K \ind{c_k\neq 0}\amultib{d}{k}\Bigg\}
\end{align*}
for generic $(A,B)$.
\end{theorem}
\begin{proof}
First, our decomposition is
\begin{align*}
    &\sum_{k=0}^K c_k \left(AB^T\right)^{(k)}\\
    &= \sum_{\substack{k=0,\ldots,K\\c_k\neq 0}} c_k \sum_{\kvec\in \compsk(k)}\binom{k}{k_1,\ldots,k_d}\left(a_1^{(k_1)}\circ\cdots\circ a_d^{(k_d)}\right)\left(b_1^{(k_1)}\circ\cdots\circ b_d^{(k_d)}\right)^T\\
    &\coloneqq \ta D \tb^T
\end{align*}
with inner indices in
\begin{align*}
    \comps(K,(c_k)) \coloneqq \{\kvec\in \Z_+^d\mid k_1+\cdots+k_d\le K,c_{k_1+\cdots+k_d}\neq 0\}.
\end{align*}
which has size
\begin{align*}
    \sum_{k=0}^K \ind{c_k\neq 0}\amultib{d}{k}.
\end{align*}
The rest of the proof proceeds as in the proof of Theorem~\ref{thm:damm}, except with $\compsk(k)$ replaced with $\comps(K,(c_k))$.
\end{proof}

\section{Results involving the Khatri-Rao product}
\label{sec:main}

Here we derive the generic ranks of $B^T\rao (AB^T)^{(k)}$, $B^T\rao \sum_{k=0}^K c_k (AB^T)^{(k)}$, $\phi(AB^T)$, and $B^T\rao \phi(AB^T)$ in Sections~\ref{sec:khatri}, \ref{sec:polykhatri}, \ref{sec:analytic}, and \ref{sec:analytic} respectively.

To do so, we will employ the fact that the non-identically zero minors of ``rectangular block diagonal'' matrices are precisely the ones corresponding to (square) block diagonal submatrices, in which case the determinant is the product of the determinants of the blocks.

\begin{lemma}
\label{lma:block}
Let $A_{\ell}\in\R^{m\times n}~\forall \ell\in[d]$ and let
\begin{align*}
    A = \begin{bmatrix}
        A_1&&\\
        &\ddots&\\
        &&A_d
    \end{bmatrix}
\end{align*}
with everything else zero. Let $I_{\ell}\subset [m(\ell-1)+1:m\ell]$ and $J_{\ell}\subset [n(\ell-1)+1:n\ell]$ for all $\ell\in[d]$. Let $I=\cupl I_{\ell}$ and $J=\cupl J_{\ell}$. Suppose $|I|=|J|\le \min\{md,nd\}$. Then $|I_{\ell}|=|J_{\ell}|$ for all $\ell\in[d]$ is necessary for $\det_{I,J}(A)$ to be nonzero and sufficient for $\det_{I,J}(A)$ to be equal to
\begin{align*}
    \prodl \det_{I_{\ell},J_{\ell}}(A_{\ell}).
\end{align*}
\end{lemma}
\begin{proof}
To prove the necessary part, suppose $|I_{\ell}|\neq |J_{\ell}|$ for some $\ell\in[d]$. If $|I_{\ell}|< |J_{\ell}|$, then the columns corresponding to $A_{\ell}$ in the $(I,J)$ submatrix are linearly dependent; if $|I_{\ell}|> |J_{\ell}|$ then the rows are; either way, $\det_{I,J}(A)=0$. The sufficient part is a well known fact and can be proved with Leibniz's determinant formula: the only permutations that give non-zero products are the ones that permute within each $A_{\ell}$ separately, and so we can factor out $\det_{I_{\ell},J_{\ell}}(A_{\ell})$ for each $\ell\in[d]$.
\end{proof}

\subsection{Hadamard powers with Khatri-Rao products}
\label{sec:khatri}

Here we derive the generic rank of $B^T\rao (AB^T)^{(k)}$.

\begin{theorem}
\label{thm:khatri}
Let $(A,B)\in\R^{m\times d}\times\R^{n\times d}$. Let $k\in\N$. Then the rank and Kruskal rank of $B^T\rao (AB^T)^{(k)}$ are $\min\{md,n,\multi{d}{k+1}\}$ for generic $(A,B)$.
\end{theorem}
\begin{proof}
Let $u\in\R^m$, $v\in\R^n$, and $W=[w_1|\cdots|w_d]\in\R^{n\times d}$. Then
\begin{align*}
    W^T \rao uv^T &= [w_{j,i_1}u_{i_2}v_j]_{(i_1,i_2),j}\\
    &=[e_1\otimes u|\cdots|e_d\otimes u][v\circ w_1|\cdots|v\circ w_d]^T\\
    &= \suml(e_{\ell}\otimes u)(v\circ w_{\ell})^T
\end{align*}
where the $e_{\ell}$ denote the basis vectors of $\R^d$. Thus,
\begin{align*}
    &B^T\rao \left(AB^T\right)^{(k)}\\
    &= B^T\rao \sum_{\kvec\in \compsk(k)}\binom{k}{k_1,\ldots,k_d}\left(a_1^{(k_1)}\circ\cdots\circ a_d^{(k_d)}\right)\left(b_1^{(k_1)}\circ\cdots\circ b_d^{(k_d)}\right)^T\\
    &= \suml\sum_{\kvec\in \compsk(k)} \binom{k}{k_1,\ldots,k_d}\left(e_{\ell}\otimes a_1^{(k_1)}\circ\cdots\circ a_d^{(k_d)}\right)\left(b_{\ell}\circ b_1^{(k_1)}\circ\cdots\circ b_d^{(k_d)}\right)^T\\
    &\coloneqq \ba D\bb^T
\end{align*}
where
\begin{align*}
    \ba = \begin{bmatrix}
        \ta&&\\&\ddots&\\&&\ta
    \end{bmatrix}\text{ with }\ta = \left[\prodl a_{i,\ell}^{k_{\ell}}\right]_{i\in[m],\kvec\in\compsk(k)},
\end{align*}
$D$ is a diagonal matrix with the corresponding multinomial coefficients, and
\begin{align*}
    \bb = \left[b_{i,j}\prodl b_{i,\ell}^{k_{\ell}}\right]_{i\in[n],(j,\kvec)\in [d]\times \compsk(k)}.
\end{align*}
Note that
\begin{align*}
    \bb=\begin{bmatrix}
        \diag(b_1)\tb&\cdots&\diag(b_d)\tb
    \end{bmatrix}\text{ where }\tb = \left[\prodl b_{i,\ell}^{k_{\ell}}\right]_{(i,\kvec)\in[n]\times\compsk(k)}.
\end{align*}
The inner dimension is $\multi{d}{k}d$.
Let $I\subset [md]$ and $J\subset [n]$ such that $|I|=|J|=s\le \min\{md,n,\multi{d}{k+1}\}$. Enumerate $I$ as $\{i_1,\ldots,i_s\}$ and $J$ as $\{j_1,\ldots,j_s\}$.
We want to show that the $(I,J)$ minor is not identically zero.
By the Cauchy-Binet formula,
\begin{align*}
    p(A,B) &\coloneqq \det_{I,J}\left(\ba D\bb^T\right)\\
    &= \sum_{S\in\binom{[d]\times\compsk(k)}{s}}\left(\prod_{\ell\in S} D_{\ell}\right)\underset{\coloneqq p_{1,S}(A)p_{2,S}(B)\coloneqq p_S(A,B)}{\underbrace{\det_{I,S}\left(\ba\right)\det_{J,S}\left(\bb\right)}}.
\end{align*}
Decompose $I$ and $S$ into $\cupl I_{\ell}$ and $\cupl S_{\ell}$ respectively, where $I_{\ell}\subset[m(\ell-1)+1:m\ell]$ and $S_{\ell}\subset[N(\ell-1)+1:N\ell]$. Then, by Lemma~\ref{lma:block}, $\det_{I,S}(\ba)$ can only be non-identically zero if $|I_{\ell}|=|S_{\ell}|$ for all $\ell\in[d]$. In this case, $\ba_{I,S}$ is block diagonal and so its determinant is the product of the determinants of its blocks, none of which are identically zero since the blocks are square submatrices of $\ta$ from the proof of Theorem~\ref{thm:damm}. Thus, $p_{1,S}$ is not identically zero if and only if $|I_{\ell}|=|S_{\ell}|$ for all $\ell\in[d]$.

On the other hand, $\bb$ has repeated columns. In particular, distinct columns $(i,\kvec)$ and $(j,\rvec)$ are equal if and only if $\kvec+e_i=\rvec+e_j$. Thus, we can break $[d]\times \compsk(k)$ up into the fibers of $\varphi:[d]\times \compsk(k)\to\!\!\!\!\to \compsk(k+1):(i,\kvec)\mapsto \kvec+e_i$. Then, $p_{2,S}$ is not identically zero if and only if $|\varphi(S)|=s$. Suppose $|\varphi(S)|=s$ and consider $p_{2,S}$. By Leibniz's determinant formula
\begin{align*}
    p_{2,S} &= \sum_{\tau:J\bij S}\sgn(\tau)\prod_{t=1}^s\prod_{\ell=1}^d b_{j_t,\ell}^{\varphi(\tau(j_t))_{\ell}}.
\end{align*}

Consider the $\tau$ Leibniz monomial:
\begin{align*}
\begin{matrix}
b_{j_1,1}^{\varphi(\tau(j_1))_{1}}&\cdots&b_{j_1,d}^{\varphi(\tau(j_1))_{d}}\\
\vdots&\ddots&\vdots\\
b_{j_s,1}^{\varphi(\tau(j_s))_{1}}&\cdots&b_{j_s,d}^{\varphi(\tau(j_s))_{d}}.
\end{matrix}
\end{align*}
Suppose it is equal to the $\tau'$ Leibniz monomial. Then $\varphi(\tau(j_t))=\varphi(\tau'(j_t))$ for each $t\in[s]$. Furthermore, since $|\varphi(S)|=s$, $\tau(j_t)=\tau'(j_t)$ for each $t\in[s]$. In other words, $\tau=\tau'$. Thus, collecting terms, the monomial $\prod_{t=1}^s\prod_{\ell=1}^d b_{j_t,\ell}^{\varphi(\tau(j_t))_{\ell}}$ has coefficient $\sgn(\tau)\neq 0$ in $p_{2,S}$.

Moreover, if the $\prod_{t=1}^s\prod_{\ell=1}^d b_{j_t,\ell}^{\varphi(\tau(j_t))_{\ell}}$ monomial has nonzero coefficient in $p_{2,S'}$ where $|\varphi(S')|=s$, then there must exist $\tau':J\bij S'$ such that it equals $\prod_{t=1}^s\prod_{\ell=1}^d b_{j_t,\ell}^{\varphi(\tau'(j_t))_{\ell}}$, which implies $S=\tau(J)=\tau'(J)=S'$.

Thus, for each $p_{2,S}$ such that $|\varphi(S)|=s$, there is a monomial, namely $\prod_{t=1}^s\prod_{\ell=1}^d b_{j_t,\ell}^{\varphi(\tau(j_t))_{\ell}}$, with nonzero coefficient in $p_{2,S}$ and zero coefficient in all other $p_{2,S'}$ with $|\varphi(S')|=s$. Thus, the $p_{2,S}(B)$ which are not identically zero are linearly independent polynomials, so the $p_{1,S}(A)p_{2,S}(B)$ which are not identically zero are linearly independent polynomials. Moreover, $p_{1,S}$ is not identically zero if and only if $|S_{\ell}|=|I_{\ell}|~\forall \ell\in [d]$ and $p_{2,S}$ is not identically zero if and only if $|\varphi(S)|=s$. Thus, all that is left to show is that there is at least one $S\in\binom{[d]\times \compsk(k)}{s}$ satisfying both these conditions.

Since $s\le \multi{d}{k+1}$, there is at least one $S$ such that $|\varphi(S)|=s$. But we need to make sure $|S_{\ell}|\le m$ so that we can pick $I_{\ell}\subset [m(\ell-1)+1:m\ell]$ such that $|I_{\ell}|=|S_{\ell}|$. We will construct such an $S$ by ``placing'' elements of $\compsk(k+1)$ into particular $S_{\ell}$.

First, we can count the number of weak compositions of $k+1$ into $d$ parts by counting the number with $i$ nonzero parts for each $i\in[d]$ and taking the sum. To count the number with $i$ nonzero parts, there are $\binom{d}{i}$ ways to choose which parts to be nonzero and $\binom{k}{i-1}$ compositions of $k+1$ into $i$ parts. Thus, $\sum_{i=1}^d\binom{d}{i}\binom{k}{i-1}=\multi{d}{k+1}$.

Now, let $i\in[d]$ and let $\alpha\subset [d]$ with $|\alpha|=i$. Let $\alpha$ be the nonzero indices of $\kvec\in\compsk(k+1)$. Then $\varphi^{-1}(\kvec)$ has $i$ elements, the $j$th being $(\alpha_j,\kvec-e_{\alpha_j})$. Thus, $\alpha$ determines which $S_{\ell}$ we can place $\kvec$ in. So, start with $i=1$. There is one $\kvec$ to place in each $S_{\ell}$. Then, continue in order for $i>1$. There are $\binom{d}{i}$ ways to choose $i$ $S_{\ell}$'s and $\binom{k}{i-1}$ $\kvec$'s to distribute among each choice. Thus, we can distribute them in such a way that, when we are done, each $S_{\ell}$ has either $\lfloor \frac{1}{d}\multi{d}{k+1}\rfloor$ or $\lceil \frac{1}{d}\multi{d}{k+1}\rceil$ elements. Finally, prune to $s$ elements total. If $md<\multi{d}{k+1}$, then we don't need to keep more than $\lfloor \frac{1}{d}\multi{d}{k+1}\rfloor$ elements from each $S_{\ell}$. If $md>\multi{d}{k+1}$, then we already have $|S_{\ell}|\le m$.

Thus, for all $J\subset [n]$ such that $|J|=s$, there is at least one $I\subset [m]$ such that $|I|=s$ and $g_{I,J}$ is not identically zero, returning to the notation of Section~\ref{sec:mainidea}. Let $s=\min\{md,n,\multi{d}{k+1}\}$. Then we have shown that
\begin{align*}
    g = \prod_{\substack{J\subset [n]\\|J|\le s}}\sum_{\substack{I\subset [md]\\ |I|=|J|}}g_{I,J}^2
\end{align*}
is not identically zero. On the other hand, all minors of order $>\multi{d}{k+1}$ are zero for arbitrary $(A,B)$ since, in this case, $\bb_{J,S}$ will have repeated columns for all $S$. The conclusion follows.
\end{proof}

\subsection{Polynomial Hadamard functions with Khatri-Rao products}
\label{sec:polykhatri}

The proof of the previous section easily extends from a monomial Hadamard function to a polynomial Hadamard function.

\begin{theorem}
\label{thm:polykhatri}
Let $(A,B)\in\R^{m\times d}\times\R^{n\times d}$. Let $K\in \N$ and let $(c_k)$ be a sequence in $\R$. Then the rank and Kruskal rank of
\begin{align*}
    B^T\rao \sum_{k=0}^K c_k \left(AB^T\right)^{(k)}
\end{align*}
are
\begin{align*}
    \min\Bigg\{md,n,\sum_{k=0}^K \ind{c_k\neq 0}\amultib{d}{k+1}\Bigg\}
\end{align*}
for generic $(A,B)$.
\end{theorem}
\begin{proof}
Here, our decomposition is
\begin{align*}
    &B^T\rao\sum_{k=0}^K c_k \left(AB^T\right)^{(k)}\\
    &= \suml \sum_{\substack{k=0,\ldots,K\\c_k\neq 0}}c_k \sum_{\kvec\in \compsk(k)} \binom{k}{k_1,\ldots,k_d}\left(e_{\ell}\otimes a_1^{(k_1)}\circ\cdots\circ a_d^{(k_d)}\right)\left(b_{\ell}\circ b_1^{(k_1)}\circ\cdots\circ b_d^{(k_d)}\right)^T\\
    &\coloneqq \ba D\bb^T
\end{align*}
with inner indices in $[d]\times \comps(K,(c_k))$ which has size
\begin{align*}
    d\sum_{k=0}^K \ind{c_k\neq 0}\amultib{d}{k}.
\end{align*}
The rest of the proof proceeds as in the proof of Theorem~\ref{thm:khatri}, except with $\compsk(k)$ replaced with $\comps(K,(c_k))$ and $\compsk(k+1)$ replaced with
\begin{align*}
    \{\kvec\in\Z_+^d\mid 1\le k_1+\cdots+k_d\le  K+1, c_{k_1+\cdots+k_d-1}\neq 0\},
\end{align*}
which has size equal to
\begin{align*}
    \sum_{k=0}^K \ind{c_k\neq 0}\amultib{d}{k+1}.
\end{align*}

\end{proof}

\subsection{Extension to real analytic Hadamard functions}
\label{sec:analytic}

Suppose $\phi:\R\to\R$ is real analytic at zero. Then it has a convergent power series expansion centered at zero. Let $c_k$ be the corresponding coefficients and let $r$ be the radius of convergence. Define $M$ as the set of $(A,B)\in\R^{m\times d}\times\R^{n\times d}$ such that the entries of $AB^T$ have absolute value less than $r$. This is clearly open. Let $I\subset [m]$ and $J\subset [n]$ such that $|I|=|J|=s\le \min\{m,n\}$. Let $K\in\N$ such that
\begin{align*}
    \sum_{k=0}^K \ind{c_k\neq 0}\amulti{d}{k}\ge \min\{m,n\}.
\end{align*}
Define
\begin{align*}
    &g_{I,J}:M\to\R:(A,B)\mapsto \det_{I,J}\left(\sum_{k=0}^{\infty} c_k \left(AB^T\right)^{(k)}\right)\\
    \text{and }&g_{I,J}^K:M\to\R:(A,B)\mapsto \det_{I,J}\left(\sum_{k=0}^{K} c_k \left(AB^T\right)^{(k)}\right).
\end{align*}
From the construction in the proof of Theorem~\ref{thm:poly}, we know that increasing $K$ does not affect the nonzero coefficients of monomials in $g_{I,J}^K$. Thus, $g_{I,J}^K$ is a truncation of the Taylor series of $g_{I,J}$ at zero. Moreover, by Theorem~\ref{thm:poly}, $g_{I,J}^K$ is not identically zero, and so it has at least one nonzero coefficient. Thus, the Taylor series of $g_{I,J}$ at zero has at least one nonzero coefficient, proving the following theorem.

\begin{theorem}
\label{thm:real}
Let $(A,B)\in\R^{m\times d}\times\R^{n\times d}$. Let $\phi:\R\to\R$. If $\phi$ is real analytic but not a polynomial, then the rank and Kruskal rank of $\phi(AB^T)$ are $\min\{m,n\}$ for generic $(A,B)$. If $\phi$ is real analytic at zero and it does not restrict to a polynomial near zero, then there exists an open neighborhood $M\subset \R^{m\times d}\times \R^{n\times d}$ of zero such that the rank and Kruskal rank of $\phi(AB^T)$ are $\min\{m,n\}$ for generic  $(A,B)\in M$.
\end{theorem}

Now let $I\subset [md]$ and $J\subset [n]$ such that $|I|=|J|=s\le \min\{md,n\}$. Let $K\in\N$ such that
\begin{align*}
    \sum_{k=0}^K \ind{c_k\neq 0}\amulti{d}{k+1}\ge \min\{md,n\},
\end{align*}
Define
\begin{align*}
    &g_{I,J}:R\to\R:(A,B)\mapsto \det_{I,J}\left(B^T\rao \phi\left(AB^T\right)\right).
\end{align*}
Proceeding similarly, we get, by Theorem \ref{thm:polykhatri}, that there is an $I$ such that $g_{I,J}$ is not identically zero, and so
\begin{align*}
    g = \prod_{\substack{J\subset [n]\\|J|\le \min\{md,n\}}}\sum_{\substack{I\subset [md]\\ |I|=|J|}}\det_{I,J}\left(B^T\rao\phi\left(AB^T\right)\right)
\end{align*}
is not identically zero. So, we have proved the following theorem.

\begin{theorem}
\label{thm:realkhatri}
Let $(A,B)\in\R^{m\times d}\times\R^{n\times d}$. Let $\phi:\R\to\R$. If $\phi$ is real analytic but not a polynomial, then the rank and Kruskal rank of $B^T\rao \phi(AB^T)$ are $\min\{md,n\}$ for generic $(A,B)$. If $\phi$ is real analytic at zero and it does not restrict to a polynomial near zero, then there exists an open neighborhood $M\subset \R^{m\times d}\times \R^{n\times d}$ of zero such that the rank and Kruskal rank of $B^T\rao \phi(AB^T)$ are $\min\{md,n\}$ for generic  $(A,B)\in M$.
\end{theorem}

\section{Application to machine learning models}
\label{sec:neuralprelims}

Here we apply the results about the rank of Khatri-Rao products of Hadamard functions to the memory capacity of machine learning models. In Section~\ref{sec:rankimplies} we show that for any model that consists of a linear model applied to a smooth embedding and that can implement scaled and shifted versions of itself by increasing its size, if there is a point where the derivative is surjective, then the model with increased size is surjective. In Section~\ref{sec:twolayer} we combine this with the fact that the Jacobian of the two layer neural network has generically full rank to show that two layer neural networks have memory capacity equal to at least half their number of parameters.

\subsection{Memory capacity from Jacobian rank}
\label{sec:rankimplies}

\newcommand{\emb}{\Bar{f}}
\newcommand{\mdl}{\Bar{F}}
\newcommand{\wemb}{\Bar{g}}
\newcommand{\wmdl}{\Bar{G}}

Consider an embedding $\emb:\R^k\times \R^d\to \R^m$, which maps feature vectors $x\in\R^d$ to embedding vectors $\hat{x}\in\R^m$ based on parameters $w\in\R^k$. Define the model $\mdl:\R^k\times \R^m\times \R^d\to \R:(w,v,x)\mapsto \langle v,\emb(w,x)\rangle$, which applies a linear model to embedding vectors. Also define the batch versions of these: $f:\R^k\times \R^{d\times n}\to \R^{m\times n}:(w,X)\mapsto [f(w,x_1)|\cdots|f(w,x_n)]$ and $F:\R^k\times\R^m\times \R^{d\times n}\to \R^n:(w,v,X)\mapsto v^Tf(w,X)$. Now, suppose that the derivative (i.e. Jacobian) with respect to the hidden layer, $\partial_w F$, is surjective at a point $(w_0,v_0,X)$, i.e. $\rank(\partial_w F(w_0,v_0,X))=n$.

Let $y\in\R^n$. Then, as an immediate consequence of the Constant Rank Theorem~\citep[Thm.~4.12]{lee2013smooth}, there are $\epsilon>0$ and $w$ near $w_0$ such that
\begin{align*}
    F(w,v_0,X) = F(w_0,v_0,X)+\epsilon\left(y-F(w_0,v_0,X)\right).
\end{align*}
Rearranging,
\begin{align*}
    y = \frac{1}{\epsilon}F(w,v_0,X)-\frac{1-\epsilon}{\epsilon}F(w_0,v_0,X).
\end{align*}
This is an interpolating solution for a scaled and shifted version of $F$, but for many model classes, it can be implemented with a slightly larger model. In particular, this is the case for two layer networks and for deep networks.

For two layer networks, we can do
\begin{align*}
    y &= \frac{1}{\epsilon}\psi\left(X^TW\right)v_0-\frac{1-\epsilon}{\epsilon}\psi\left(X^TW_0\right)v_0\\
    &=\psi\left(X^T\begin{bmatrix}W&W_0\end{bmatrix}\right)[v_0/\epsilon;-(1-\epsilon)v_0/\epsilon]
\end{align*}
(where we've changed notation from $w\in\R^k$ to $W\in\R^{d\times m}$).

For the $(L+1)$-layer network
\begin{align*}
    \psi\left(\cdots \psi\left(X^TW_1\right)\cdots W_L\right)v
\end{align*}
(where we've changed notation from $w\in\R^k$ to $(W_1,\ldots,W_L)\in \R^{m_0\times m_1}\times\cdots\times\R^{m_{L-1}\times m_L}$ with $m_0=d$), we can do
\begin{align*}
    y &= \frac{1}{\epsilon}\psi\left(\cdots \psi\left(X^TW_1\right)\cdots W_L\right)v_0-\frac{1-\epsilon}{\epsilon}\psi\left(\cdots \psi\left(X^TW_{1,0}\right)\cdots W_{L,0}\right)v_0\\
    &= \psi\left(\cdots \psi\left(\left(X^T\begin{bmatrix}W_1&W_{1,0}\end{bmatrix}\right)\begin{bmatrix}W_2&\\&W_{2,0}\end{bmatrix}\right)\cdots \begin{bmatrix}W_L&\\&W_{L,0}\end{bmatrix}\right)\begin{bmatrix}v_0/\epsilon\\-(1-\epsilon)v_0/\epsilon\end{bmatrix}.
\end{align*}

For more general models, we can apply Theorem~\ref{thm:everything}, which relies on the following immediate consequence of the Constant Rank Theorem.

\begin{lemma}
\label{prop:constantrank}
Let $M\subset \R^k$ be open. Let $F:M\to\R^n$ be $C^1$. If there exists $x_0\in M$ such that $\rank(DF(x_0))=n$, then for all $y\in\R^n$ there are $\epsilon>0$ and $x\in M$ such that $F(x)=F(x_0)+\epsilon(y-F(x_0))$.
\end{lemma}
\begin{proof}
First, since $n$ is the highest possible rank, $\{x\in M\mid \rank(DF(x))=n\}$ is the preimage of $\R\backslash\{0\}$ under the rank condition function (the sum of the squares of the minors of order $n$). Moreover, the rank condition function is continuous since $F$ is $C^1$, so $\{x\in M\mid \rank(DF(x))=n\}$ is open. Thus, $\rank(DF(x))=n$ for all $x$ in an open neighborhood $A\subset M$ of $x_0$. In other words, $F|_A$ is a $C^1$-submersion, so, by the Constant Rank Theorem \citep[Thm.~4.12]{lee2013smooth}, there is an open neighborhood $A'\subset A$ of $x_0$ such that $F(A')$ is open, from which the conclusion follows.
\end{proof}

\begin{theorem}
\label{thm:everything}
Let $M\subset \R^k$ be open. Let $f:M\to\R^{n\times m}$ be $C^1$. Define $g:M\times M\to\R^{n\times 2m}:(w,u)\mapsto [f(w)~f(u)]$. For all $v,z\in\R^m$, define $F_v:M\to\R^n:w\mapsto f(w)v$ (where $f(w)v$ is a matrix multiplication) and $G_{v,z}:M\times M\to\R^n:(w,u)\mapsto g(w,u)[v;z]$. If there exists $v_0\in\R^m$ and $w_0\in M$ such that $\rank(DF_{v_0}(w_0))=n$, then $G$ is surjective as a function of $(v,z)\in\R^m\times\R^m$ and $(w,u)\in M\times M$.
\end{theorem}
\begin{proof}
Let $y\in \R^n$. Then, by Lemma \ref{prop:constantrank}, there are $\epsilon>0$ and $w\in M$ such that $F_{v_0}(w)=F_{v_0}(w_0)+\epsilon(y-F_{v_0}(w_0))$. Rearranging,
\begin{align*}
    y &= \frac{1}{\epsilon}F_{v_0}(w)-\frac{1-\epsilon}{\epsilon}F_{v_0}(w_0)\\
    &= G_{v_0/\epsilon,-(1-\epsilon)v_0/\epsilon}(w,w_0).
\end{align*}
\end{proof}

\subsection{Memory capacity of two layer neural network}
\label{sec:twolayer}

In Theorems~\ref{thm:polykhatri} and~\ref{thm:realkhatri} we derived the generic rank of the Jacobian of the two layer neural network, with respect to the hidden layer, for real analytic activation functions. In Theorem~\ref{thm:everything} we showed that for many models, including two layer neural networks, surjectivity of the derivative at a point implies surjectivity of the model with increased size. Combining these, we get the following result about the memory capacity of two layer neural networks.

\begin{theorem}
\label{thm:ae}
Define $h:\R^{d\times m}\times \R^m\times\R^m\times \R^{d\times n}\to\R^n:(W,b,v,X)\mapsto \psi(X^TW+\one_n  b^T)v$ where $\psi:\R\to\R$ is applied coordinate-wise. Suppose $\psi$ is real analytic at some point $\eta$ and, if $\psi(x)=\sum_{k=0}^Kc_kx^k$ near $\eta$, assume $\sum_{k=1}^K\ind{c_k\neq 0}\multi{d}{k}\ge n$. If $md\ge 2n$ and $m$ is even, then $h(\cdot,\cdot,\cdot,X)$ is surjective for all $X$ except on a measure zero and closed set.
\end{theorem}
\begin{proof}
Recall that the transpose of the Jacobian of $h$ with respect to $\tvec(W)$\\is $\diag(v)\psi'\left(W^TX+b \one_n^T\right)\rao X$. Let $\phi=\psi'-\eta$. By Theorem~\ref{thm:polykhatri} or~\ref{thm:realkhatri}, there exists an open neighborhood $M\subset\R^{d\times m/2}\times \R^{d\times n}$ of zero such that $\rank\left(\phi(W^TX)\rao X\right)=n$ for generic $(W,X)\in M$. Thus, for generic $X\in \R^{d\times n}$, there exists $W_0\in\R^{d\times m/2}$ (close to zero) such that $\rank\left(\phi(W_0^TX)\rao X\right)=n$. Setting $v_0\in\R^{m/2}$, applying Theorem~\ref{thm:everything}, and setting $b=\eta\one_{m}$ proves the result.
\end{proof}

Note that while the solution from the proof is implicit rather than constructive, it does suggest a method for checking whether a particular $X$ is on the measure zero and closed set or not. The result holds for all $X\in\R^{d\times n}$ such that there exists $W_0\in\R^{d\times m/2}$ (close to zero) such that $\rank(\phi(W_0^TX)\rao X)=n$. So, given $X$, we can initialize $W_0$ with i.i.d. standard normal entries, choose $\rho>0$ so that the entries of $W_0^TX/\rho$ are within the interval of convergence of $\phi$ at zero, and compute the rank of $\phi(W_0^TX/\rho)\rao X$. If the rank is $n$, then we get that $h(\cdot,\cdot,\cdot,X)$ is surjective. Thus, if we plan to find an interpolating solution using an iterative algorithm, such as gradient descent, we only have to check the rank of the Jacobian at initialization to determine whether an interpolating solution exists.

Also note that the additional assumption when $\psi$ restricts to a polynomial is actually quite mild. If $\psi$ is a degree $K$ polynomial with all of its coefficients nonzero, then the assumption becomes $\binom{K+d}{d}\ge n+1$ by Zhū's Theorem (see Section~\ref{sec:compositions}). Or, if $\psi$ is the monomial $x^K$, then the assumption becomes $\binom{K+d-1}{d-1}\ge n$. As an example, the assumption is satisfied if $d=1e2$, $n=1e5$, and $\psi=x^3$.

Also, for continuously differentiable activations, Theorem~\ref{thm:ae} is actually tight up to a factor of $2(1+2/d)\approx 2$ by Sard's theorem.

\begin{theorem}
\label{thm:lowerbound}
Define $h:\R^{d\times m}\times \R^m\times\R^m\times \R^{d\times n}\to\R^n:(W,b,v,X)\mapsto \psi(X^TW+\one_nb^T)v$ where $\psi$ is a continuously differentiable function applied coordinate-wise.
If $m(d+2)< n$ or $\psi(x)=\sum_{k=0}^Kc_kx^k$ and $\sum_{k=1}^K\ind{c_k\neq 0}\multi{d}{k}< n-2m$, then, for all $X$, the image of $h(\cdot,\cdot,\cdot,X)$ has measure zero in $\R^n$.
\end{theorem}
\begin{proof}
Let $X\in\R^{d\times n}$. If $m(d+2)<n$, then every point in $\R^{d\times m}\times \R^m\times\R^m$ is a critical point of $h(\cdot,\cdot,\cdot,X)$, i.e. a point where the differential is not surjective \citep[p. 105]{lee2013smooth}. If $m(d+2)\ge n$, $\psi(x)=\sum_{k=0}^Kc_kx^k$, and $\sum_{k=1}^K\ind{c_k\neq 0}\multi{d}{k}< n-2m$, then, by Theorem~\ref{thm:polykhatri}, the rank of the Jacobian of $h(\cdot,\cdot,\cdot,X)$ is upper bounded by $2m+\sum_{k=1}^K\ind{c_k\neq 0}\multi{d}{k}<n$ and so, again, every point in $\R^{d\times m}\times\R^m\times \R^m$ is a critical point of $h(\cdot,\cdot,\cdot,X)$. Thus, every point in the image of $h(\cdot,\cdot,\cdot,X)$ is a critical value of $h(\cdot,\cdot,\cdot,X)$, i.e. a value whose level set (preimage/fiber) contains at least one critical point. And so, by Sard's theorem~\citep[Thm. 6.10]{lee2013smooth}, which states that the set of critical values has measure zero, the image of $h(\cdot,\cdot,\cdot,X)$ has measure zero in $\R^n$.
\end{proof}

As a final remark, Theorems~\ref{thm:ae} and~\ref{thm:lowerbound} have corollaries if we consider $h:\R^{d\times m}\times \R^m\times\R^{m\times q}\times \R^{d\times n}\to\R^{n\times q}:(W,b,V,X)\mapsto \psi(X^TW+\one_nb^T)V$, i.e. if we consider multivariate outputs. First, if $m(d+q+1)<nq$, then, by Sard's theorem, the image of $h(\cdot,\cdot,\cdot,X)$ has measure zero in $\R^{n\times q}$. On the other hand, assume $md\ge 2nq$ and that $q|m$. Let $Y=[y_1|\cdots|y_q]\in\R^{n\times q}$. Then, for each $i\in [q]$, by Theorem~\ref{thm:ae} (assuming $\psi$ is real analytic at a point and not a polynomial there), there exist $W_i\in\R^{d\times m/q}$ and $v_i,b_i\in\R^{m/q}$ such that $y_i=\psi(X^TW_i+\one_nb_i^T)v_i$. Thus,
\begin{align*}
    Y = \psi\left(X^T\left[W_1~\cdots~W_q\right]+\one_n\left[b_1^T~\cdots~b_q^T\right]\right)\begin{bmatrix}
    v_1&&\\
    &\ddots&\\
    &&v_q
    \end{bmatrix},
\end{align*}
achieving interpolation with $2(1+q/d+1/d)$ times the minimal number of parameters.

But, if $2q>d$, then we can actually get interpolation with fewer parameters by solving for $V$ instead of $W$. If $m\ge n$ and $\psi$ is real analytic at $\eta$ and not a polynomial there, then $\psi(X^TW+\eta\one_n\one_m^T)$ has rank $n$ by Theorem~\ref{thm:real}. Thus, we can explicitly solve for $V$ in the equation $Y=\psi(X^TW+\eta\one_n\one_m^T)V$, achieving interpolation with $1+d/q+1/q$ times the minimal number of parameters.

Putting these two cases together, the memory capacity of $h$ is optimal up to a factor of $\min\{2,d/q\}(1+q/d+1/d)\le 3+2/d\approx 3$.

\section{Conclusion}
\label{sec:conclusion}

In this paper we proved results about the ranks of both $\phi(W^TX)$ and $\phi(W^TX)\rao X$ for Hadamard functions $\phi$ that are real analytic at a point. In particular, if $\phi$ is either a polynomial with sufficiently high degree or not a polynomial, then both objects are full rank for all $(W,X)$ except on the zero set of a non-identically zero analytic function, and so the set of exceptions is measure zero and closed. In the context of machine learning, this rank result implies that a two layer neural network model is able to interpolate sets of $n$ data points almost everywhere if the number of trainable parameters is $\ge 2n(1+2/d)$, which is tight up to the factor of $2(1+2/d)$. We hope that viewing the rank of the neural network Jacobian from the perspective of analytic functions of several variables leads to a better understanding of the convergence and generalization properties of two layer neural networks.

\section*{Acknowledgments}
We would like to thank Jason Altschuler for pointing out the references \cite{alon2009perturbed} and \cite{barvinok2012approximations}. We would like to thank the authors of \cite{zhang2021when} for helpful discussions concerning their findings. And we would like to thank an anonymous reviewer for encouraging us to try to extend beyond real analytic activations: this lead us to realize we only needed the activation to be real analytic at a point.

\bibliographystyle{spbasic}
\bibliography{references}

\appendix

\section{Non-polynomial real analytic activations}
\label{sec:zhang}

After corresponding with the authors of \cite{zhang2021when}, we came to understand that the following was the idea behind the proof of their Lemma E.1 (typos on top of convoluted presentation had previously prevented us from gleaning this):

Let $\phi$ be non-polynomial real analytic. Suppose $d\mid n$. Let $A=[a_1|\cdots|a_d]\in\R^{m\times d}$ and
\begin{align}
\label{eq:bform}
    B =\begin{bmatrix}
        b_1&&\\&\ddots&\\&&b_d
    \end{bmatrix}
\end{align}
where each $b_i\in\R^{n/d}$. Then
\begin{align*}
    B^T\rao \phi(AB^T) &= B^T\rao \phi\left([a_1b_1^T|\cdots|a_db_d^T]\right)\\
    &=\begin{bmatrix}
        b_1^T&&\\&\ddots&\\&&b_d^T
    \end{bmatrix}\rao \left[\phi\left(a_1b_1^T\right)|\cdots|\phi\left(a_db_d^T\right)\right]\\
    &=\begin{bmatrix}
        b_1^T\rao\phi\left(a_1b_1^T\right) &&\\&\ddots&\\&&b_d^T\rao\phi\left(a_db_d^T\right)
    \end{bmatrix}.
\end{align*}

Since generic $b_i$ have no zeros, the problem comes down to the generic rank of the $\phi(a_ib_i^T)$. \cite{zhang2021when} use that the $\phi(a_ib_i^T)$ have generically full rank if $\phi$ is not a polynomial. This is a corollary of Theorem~\ref{thm:real}, which has been known for some time.

While \cite{zhang2021when} do not explicitly consider polynomials, we can extend to polynomials by applying Theorem~\ref{thm:poly}, which has also been known for some time. Consider $\phi=\sum_{k=0}^Kc_kx^k$. Note that $d=1$ here and so, by Theorem~\ref{thm:poly}, the generic rank of $\phi(a_ib_i^T)$ is precisely
\begin{align*}
    \min\bigg\{m,\frac{n}{d},\sum_{k=0}^{K}\ind{c_k\neq 0}\bigg\}.
\end{align*}
Thus, the generic rank of $B^T\rao \phi(AB^T)$ where $B$ has the form of Eq.~\eqref{eq:bform} is exactly
\begin{align}
\label{eq:zhang}
    \min\bigg\{md,n,d\sum_{k=0}^{K}\ind{c_k\neq 0}\bigg\},
\end{align}
which is much smaller than the generic rank which we derived in Theorem~\ref{thm:polykhatri}. In particular, suppose $md\ge n$. If all the $c_k$ are nonzero, then $K$ has to be greater than or equal to $n/d-1$ in order for Eq.~\eqref{eq:zhang} to equal $n$. Furthermore, if $\phi$ is a monomial, such as $x^K$, then Eq.~\eqref{eq:zhang} will not equal $n$ for any $K\in\N$.

To compare this to our Theorem~\ref{thm:polykhatri}, consider the example $d=1e2$, $n=1e5$, and $m\ge 1e3$: $\phi=x^3$ is sufficient for the generic rank in Theorem~\ref{thm:polykhatri} to be $n$, while it is necessary that $\phi$ have at least 1000 nonzero coefficients for Eq.~\eqref{eq:zhang} to be $n$.

\section{Inner product decompositions}
\label{sec:directsum}

Using the result from Section~\ref{sec:product}, we have an immediate corollary. If the rows of $\widetilde{A}\in\R^{m\times N}$ and $\widetilde{B}\in\R^{n\times N}$ are arbitrary vectors in a $d$-dimensional subspace $V\subset \R^N$, then the generic rank of $\widetilde{A}\widetilde{B}^T$ is the same as the generic rank of $AB^T$. In particular, we can apply this to the following decomposition from Lemma 4.4 of \cite{barvinok2012approximations},
\begin{align*}
    \sum_{k=0}^Kc_k(AB^T)^{(k)}&= \begin{bmatrix}
        \bigoplus\limits_{\substack{k=0\\c_k\neq 0}}^Kc_ka_1^{\otimes k}\bigg|\cdots\bigg|\bigoplus\limits_{\substack{k=0\\c_k\neq 0}}^Kc_ka_m^{\otimes k}
    \end{bmatrix}^T
    \begin{bmatrix}
        \bigoplus\limits_{\substack{k=0\\c_k\neq 0}}^Kb_1^{\otimes k}\bigg|\cdots\bigg|\bigoplus\limits_{\substack{k=0\\c_k\neq 0}}^Kb_n^{\otimes k}
    \end{bmatrix}
\end{align*}
where $A=[a_1|\cdots|a_m]^T$ and $B=[b_1|\cdots|b_n]^T$. The vectors in this decomposition are arbitrary vectors in
\begin{align*}
    \bigoplus\limits_{\substack{k=0\\c_k\neq 0}}^K\text{Sym}^k(\R^d)
\end{align*}
where $\text{Sym}^0(\R^d)=\R$ and, for $k\in\N$, $\text{Sym}^k(\R^d)$ is the vector space of symmetric tensors of order $k$ defined on $\R^d$. The dimension of this space is
\begin{align*}
    \sum\limits_{\substack{k=0\\c_k\neq 0}}^K\amulti{d}{k},
\end{align*}
proving our Theorem~\ref{thm:poly}.

It turns out we can construct a similar decomposition for $B^T\rao \sum_{k=0}^K c_k (AB^T)^{(k)}$. Observe that the $((i_1,i_2),j)$ element is
\begin{align*}
    \langle e_{i_1},b_j\rangle\sum_{k=0}^K  c_k\langle a_{i_2},b_j\rangle^k &= \langle e_{i_1},b_j\rangle\sum_{k=0}^K  c_k\langle a_{i_2}^{\otimes k},b_j^{\otimes k}\rangle\\
    &=\langle e_{i_1},b_j\rangle\Bigg\langle\bigoplus\limits_{\substack{k=0\\c_k\neq 0}}^Kc_ka_{i_2}^{\otimes k},\bigoplus\limits_{\substack{k=0\\c_k\neq 0}}^Kb_j^{\otimes k}\Bigg\rangle\\
    &=\Bigg\langle e_{i_1}\otimes \bigoplus\limits_{\substack{k=0\\c_k\neq 0}}^Kc_ka_{i_2}^{\otimes k},\bigoplus\limits_{\substack{k=0\\c_k\neq 0}}^Kb_j^{\otimes (k+1)}\Bigg\rangle
\end{align*}
and so we can write
\begin{align}
\label{eq:barv}
    B^T\rao \sum_{k=0}^K c_k (AB^T)^{(k)}=\begin{bmatrix}
        \widetilde{A}&&\\
        &\ddots&\\
        &&\widetilde{A}
    \end{bmatrix}^T
    \begin{bmatrix}
        \bigoplus\limits_{\substack{k=0\\c_k\neq 0}}^Kb_1^{\otimes (k+1)}\bigg|\cdots\bigg|\bigoplus\limits_{\substack{k=0\\c_k\neq 0}}^Kb_n^{\otimes (k+1)}
    \end{bmatrix}
\end{align}
where
\begin{align*}
    \widetilde{A} = \begin{bmatrix}
        \bigoplus\limits_{\substack{k=0\\c_k\neq 0}}^Kc_ka_1^{\otimes k}\bigg|\cdots\bigg|\bigoplus\limits_{\substack{k=0\\c_k\neq 0}}^Kc_ka_m^{\otimes k}
    \end{bmatrix},
\end{align*}
but this does not reduce to $AB^T$.

\end{document}